\documentclass[journal]{IEEEtran}
\usepackage{times}
\usepackage{cite}
\usepackage{helvet}
\usepackage{courier}
\usepackage{booktabs}
\usepackage{floatrow}
\usepackage{url}
\usepackage{floatrow}
\usepackage{graphicx} 
\usepackage{epsfig}
\usepackage{amssymb}
\usepackage{dsfont}
\usepackage{algorithm}
\usepackage{subfigure}
\usepackage[mathlines]{lineno}
\usepackage{amsmath}
\usepackage{subfigure}
\usepackage{multirow}
\usepackage{amsthm}
\newtheorem{theorem}{Theorem}
\newtheorem{corollary}{Corollary}

\newtheorem{proposition}{Proposition}
\usepackage[colorlinks, linkcolor=black, urlcolor=blue,anchorcolor=black, citecolor=blue]{hyperref}

\ifCLASSOPTIONcompsoc
  \usepackage[nocompress]{cite}
\else
  \usepackage{cite}
\fi
\ifCLASSINFOpdf
\else
\fi
\begin{document}
%
\title{Connections Between Nuclear Norm and Frobenius Norm Based Representations}
%
%
%
%

\author{Xi Peng,
        Canyi Lu,		
        Zhang Yi,~\IEEEmembership{Fellow~IEEE,}
        Huajin Tang,~\IEEEmembership{Member~IEEE}
\thanks{This work was supported by  A*STAR Industrial Robotics Programme - Distributed Sensing and Perception under SERC grant 1225100002, the National Natural Science Foundation of China under Grant 61432012 and 61673283. Corresponding author: H. Tang.}
\thanks{X. Peng is with Institute for Infocomm Research, A*STAR, Singapore 138632 (E-mail: pangsaai@gmail.com).}
\thanks{C. Lu is with Department of Electrical and Computer Engineering at National University of Singapore, Singapore 119077. (E-mail: canyilu@gmail.com).}
\thanks{Z. Yi and H. Tang are with College of Computer
Science, Sichuan University, Chengdu, China 610065 (E-mail: \{zhangyi,htang\}@scu.edu.cn). 
}
}

%
%

\markboth{}%
{Shell \MakeLowercase{\textit{et al.}}: Bare Demo of IEEEtran.cls for Computer Society Journals}
%



\IEEEtitleabstractindextext{%
\begin{abstract}
A lot of works have shown that frobenius-norm based representation (FNR) is competitive to sparse representation and nuclear-norm based representation (NNR) in numerous tasks such as subspace clustering. Despite the success of FNR in experimental studies, less theoretical analysis is provided to understand its working mechanism. In this paper, we fill this gap by building the theoretical connections between FNR and NNR. More specially, we prove that: 1) when the dictionary can provide enough representative capacity, FNR is exactly NNR even though the data set contains the Gaussian noise, Laplacian noise, or sample-specified corruption; 2) otherwise, FNR and NNR are two solutions on the column space of the dictionary. 
\end{abstract}

}

\maketitle

\IEEEdisplaynontitleabstractindextext

%
\IEEEpeerreviewmaketitle

\section{Introduction}\label{sec1}

\IEEEPARstart{M}{any} problems in machine learning and computer vision begin with the processing of linearly inseparable data. The goal of processing is to distinct linearly inseparable data with linear methods. To achieve this, the inputs are always projected from the original space into another space. This is so-called representation learning and three methods have been  extensively investigated in the community of computer vision, \textit{i.e.}, sparse representation (SR), low rank representation (LRR), and Frobenius-norm based representation (FNR).

During the past decade, sparse representation~\cite{Donoho2003,Wright2010} has been one of the most popular representation learning methods. It linearly reconstructs each sample using a few of basis and has shown the effectiveness in a lot of applications, \textit{e.g.}, image repairing~\cite{Aharon2006}, face recognition~\cite{Wright2009}, online learning control~\cite{Xu2013:TNNLS}, dimension reduction~\cite{Cheng2010}, and subspace clustering~\cite{Elhamifar2013,Peng2015SRSC}. 

 As another popular method, low rank representation~\cite{Liu2013,Favaro2011,Vidal2014,Sprechmann2015,Xiao2015TNN,Xiao2015aTNN} has been proposed for subspace learning and subspace clustering. Different from SR, LRR computes the representation of a data set rather than a data point by solving a nuclear norm minimization problem. Thus, LRR is also known as nuclear norm based representation (NNR). Both LRR and SR benefit from the compressive sensing  theory~\cite{Cai2010} which establishes the equivalence between $\ell_0$- (rank-minimization \textit{w.r.t.} matrix space) and $\ell_1$-norm (nuclear-norm \textit{w.r.t.} matrix space) based optimization problems. More specifically, compressive sensing provides the theoretical foundation to transform the non-convex problem caused by $\ell_0$-norm into a convex problem using $\ell_1$-norm. 
 
Recently, several works have shown that the Frobenius norm based representation (FNR) is competitive to SR and NNR in face recognition~\cite{Naseem2010,Shi2011face,Zhang2011}, subspace learning~\cite{Peng2012,Peng2016:Auto}, feature selection~\cite{Xu2016:TNNLS}, and subspace clustering~\cite{Lu2012,Peng2015robust}. The advantage of FNR is that the objective only involves a strictly convex problem and thus the trap of local minimal is avoided. 

Although more and more experimental evidences have been provided to show the effectiveness of FNR, the success of FNR is counter-intuitive as FNR is generally considered to be inferior to SR and NNR. Furthermore, fewer theoretical studies have been done to explore what makes FNR effective. Motivated by two NNR works~\cite{Favaro2011,Vidal2014}, this paper provides a novel theoretical explanation by bridging FNR and NNR. In other words, we show that under some mild conditions, the convex problem caused by nuclear norm can be converted to a strictly convex problem based on the Frobenius norm. More specifically, we prove that: 1) when the dictionary has enough representative capacity, FNR is equivalent to the NNR~\cite{Favaro2011,Vidal2014} even though the data set contains the Gaussian noise, Laplacian noise, or sample-specified corruption; 2) when the dictionary has limited representative capacity, FNR and NNR are two solutions of the column space spanned by inputs. Our theoretical results unify FNR and NNR into a framework, \textit{i.e.}, FNR and NNR are in the form of $\mathbf{V}\mathcal{P}(\mathbf{\Sigma})\mathbf{V}^{T}$, where $\mathbf{U}\mathbf{\Sigma}\mathbf{V}^{T}$ is the singular value decomposition (SVD) of a given data matrix and $\mathcal{P}(\cdot)$ denotes the shrinkage-thresholding operator. The difference between FNR and NNR lies in the different choices of the shrinkage-thresholding operator. To the best of our knowledge, this is one of the first several works to establish the connections between FNR and NNR.

\section{Background}
\label{sec2}

For a given data set $\mathbf{X}\in \mathds{R}^{m\times n^{\prime}}$ (each column denotes a data point), it can be decomposed as the linear combination of $\mathbf{D}\in \mathds{R}^{m\times n}$ by
\begin{equation}
\label{eq2.1}
\mathop{\min}_{\mathbf{C}} f(\mathbf{C})\hspace{3mm} \mathrm{s.t.} \hspace{1mm} \mathbf{X}=\mathbf{DC},
\end{equation}
where $f(\mathbf{C})$ denotes the constraint enforced over the representation $\mathbf{C}\in \mathbf{R}^{n\times n^{\prime}}$. The main difference among most existing works is their objective functions, basically, the choice of $f(\mathbf{C})$. Different assumptions motivate different $f(\cdot)$ and this work focuses on the discussion of two popular objective functions, \textit{i.e.}, nuclear-norm and Frobenius-norm. 

By assuming $\mathbf{C}$ is low rank and the input contains noise, Liu \textit{et\ al.}~\cite{Liu2013} propose solving the following nuclear norm based minimization problem:
\begin{equation}
\label{eq2.4}
\mathop{\min}_{\mathbf{C},\mathbf{E}}\hspace{1mm}\underbrace{\|\mathbf{C}\|_{\ast}+\lambda\|\mathbf{E}\|_{p}}_{f(\mathbf{C})}
\hspace{3mm} \mathrm{s.t.} \hspace{1mm} \underbrace{\mathbf{D}=\mathbf{DC}+\mathbf{E}}_{\mathrm{Noisy\ Case}},
\end{equation}
where $\|\mathbf{C}\|_{\ast}=\sum\sigma_{i}(\mathbf{C})$, $\sigma_i(\mathbf{C})$ is the $i$th singular value of $\mathbf{C}$, and $\|\cdot\|_{p}$ could be chosen as $\ell_{2,1}$-, $\ell_{1}$-, or Frobenius-norm. $\ell_{2,1}$-norm is usually adopted to depict the sample-specific corruptions such as outliers, $\ell_1$-norm is used to characterize the Laplacian noise, and Frobenius norm is used to describe the Gaussian noise. 

Although Eq.(\ref{eq2.4}) can be easily solved by the Augmented Lagrangian method (ALM)~\cite{Boyd2011}, its computational complexity is still very high. Recently, Favaro and Vidal~\cite{Favaro2011,Vidal2014} proposed a new formulation of LRR which can be calculated very fast. The proposed objective function is as follows:
\begin{equation}
\label{eq2.5}
\mathop{\min}_{\mathbf{C},\mathbf{D}_{0}}\|\mathbf{C}\|_{\ast}+\lambda\|\mathbf{D}-\mathbf{D}_{0}\|_{F}
\hspace{3mm} \mathrm{s.t.} \hspace{1mm} \mathbf{D}=\mathbf{D}_{0}\mathbf{C}+\mathbf{E},
\end{equation}
where $\mathbf{D}_{0}$ denotes the clean dictionary and $\|\cdot\|_{F}$ denotes the Frobenius-norm of a given data matrix. Different from Eq.(\ref{eq2.4}), Eq.(\ref{eq2.5}) calculates the low rank representation using a clean dictionary $\mathbf{D}_{0}$ instead of the original data $\mathbf{D}$. Moreover,  Eq.(\ref{eq2.5}) has a closed-form solution. In this paper, we mainly investigate this formulation of NNR.

Another popular representation is based on $\ell_2$-norm or its induced matrix norm (\textit{i.e.}, the Frobenius norm). The basic formulation of FNR is as follows:
\begin{equation}
\label{eq2.6}
\min \|\mathbf{C}\|_{F}\hspace{3mm} \mathrm{s.t.} \hspace{1mm} \mathbf{X}=\mathbf{DC}.
\end{equation}

In our previous work~\cite{Zhang2014}, we have shown that the optimal solution to Eq.(\ref{eq2.6}) is also the lowest rank solution, \textit{i.e.}, 

\begin{theorem}[\cite{Zhang2014}]
\label{thm2}
Assume $\mathbf{D}\ne \mathbf{0}$ and $\mathbf{X}=\mathbf{DC}$ has feasible solution(s), \textit{i.e.}, $\mathbf{X}\in span(\mathbf{D})$. Then
\begin{equation}
	\label{lem2:eq1}
	\mathbf{C}^{\ast}=\mathbf{D}^{\dag}\mathbf{X}
\end{equation}
is the unique minimizer to Eq.(\ref{eq2.6}), where $\mathbf{D}^{\dag}$ is the pseudo-inverse of $\mathbf{D}$.
\end{theorem}

Considering nuclear norm based minimization problem, Liu \textit{et\ al.}~\cite{Liu2013} have shown that 
\begin{theorem}[\cite{Liu2013}]
\label{thm1}
Assume $\mathbf{D}\ne \mathbf{0}$ and $\mathbf{X}=\mathbf{DC}$ has feasible solution(s), \textit{i.e.}, $\mathbf{X}\in span(\mathbf{D})$. Then
\begin{equation}
	\label{lem1:eq1}
	\mathbf{C}^{\ast}=\mathbf{D}^{\dag}\mathbf{X}
\end{equation}
is the unique minimizer to
\begin{equation}
	\min \|\mathbf{C}\|_{\ast}\hspace{3mm} \mathrm{s.t.}\hspace{1mm}\mathbf{X}=\mathbf{D}\mathbf{C},
\end{equation}
where $\mathbf{D}^{\dag}$ is the pseudo-inverse of $\mathbf{D}$.
\end{theorem}

Theorems~\ref{thm2} and~\ref{thm1} actually imply the equivalence between NNR and FNR when the dictionary can exactly reconstruct inputs and the data set  is immune from corruptions. In this paper, we will further investigate the connections between NNR and FNR by considering more complex situations, \textit{e.g.}, the data set is corrupted by Gaussian noise. 


\begin{table*}[t]
\caption{Connections between nuclear norm ($\|\mathbf{C}\|_{\ast}\triangleq \sum_{i}\sigma_{i}(\mathbf{C})$) and Frobenius norm ($\|\mathbf{C}\|_{F}^{2}\triangleq \sum_{i}\sigma_{i}^{2}(\mathbf{C})$) based representation in the case of \textit{the noise-free} and \textit{the Gaussian noise} situations, where $\sigma_{i}(\mathbf{C})$ denotes the $i$th singular value of $\mathbf{C}$. $\mathbf{D}=\mathbf{U}\mathbf{\Sigma}\mathbf{V}^{T}$ is the full SVD of the dictionary $\mathbf{D}$ and $\mathbf{\Sigma}=diag(\sigma_{1}, \sigma_{2}, \cdots)$.}
\label{tab1}
\begin{center}
\begin{footnotesize}
\begin{tabular}{cccc}
\toprule
Objective Function & $\mathbf{C}^{\ast}$ & $\mathcal{P}_{k}(\sigma_{i})$ or $\mathcal{P}_{\gamma}(\sigma_{i})$ & $k$ or $\omega_{i}$\\
\midrule
$\min \|\mathbf{C}\|_{F}\hspace{3mm}\mathrm{s.t.}\hspace{1mm}\mathbf{X}=\mathbf{D}\mathbf{C}$ & $\mathbf{D}^{\dag}\mathbf{X}$ & Nil
   & Nil\\
$\min \|\mathbf{C}\|_{\ast}\hspace{3mm}\mathrm{s.t.}\hspace{1mm}\mathbf{X}=\mathbf{D}\mathbf{C}$ & $\mathbf{D}^{\dag}\mathbf{X}$ & Nil
   & Nil\\
$\min \|\mathbf{C}\|_{F}\hspace{3mm}\mathrm{s.t.}\hspace{1mm}\mathbf{D}=\mathbf{D}\mathbf{C}$ & $\mathbf{V}\mathcal{P}_{k}(\mathbf{\Sigma})\mathbf{V}^{T}$ & $
  \left\{
   \begin{aligned}
   1 &\hspace{0.3cm} i \leq k  \\
   0 &\hspace{0.3cm} i > k  \\
   \end{aligned}
   \right.$
   & $k=rank(\mathbf{D})$\\
$\min \|\mathbf{C}\|_{\ast}\hspace{3mm}\mathrm{s.t.}\hspace{1mm}\mathbf{D}=\mathbf{D}\mathbf{C}$ & $\mathbf{V}\mathcal{P}_{k}(\mathbf{\Sigma})\mathbf{V}^{T}$ & $
  \left\{
   \begin{aligned}
   1 &\hspace{0.3cm} i \leq k  \\
   0 &\hspace{0.3cm} i > k  \\
   \end{aligned}
   \right.$
   & $k=rank(\mathbf{D})$\\
   $\min \frac{1}{2}\|\mathbf{C}\|_{F}^{2}+\frac{\lambda}{2}\|\mathbf{D}-\mathbf{D}_{0}\|_{F}^{2}
\hspace{3mm} \mathrm{s.t.} \hspace{1mm} \mathbf{D}_{0}=\mathbf{D}_{0}\mathbf{C}$ & $\mathbf{V}\mathcal{P}_{k}(\mathbf{\Sigma})\mathbf{V}^{T}$ & $  \left\{
   \begin{aligned}
   1 &\hspace{0.3cm} i \leq k  \\
   0 &\hspace{0.3cm} i > k  \\
   \end{aligned}
   \right.$
   & $\left\{
   \begin{aligned}
   	k&=\mathrm{argmin}_{r} r+\lambda\sum_{i>r}\sigma_{i}^{2}\\
   	r &= rank(\mathbf{D}_{0})\\
   \end{aligned}\right.$\\
   $\min \|\mathbf{C}\|_{\ast}+\frac{\lambda}{2}\|\mathbf{D}-\mathbf{D}_{0}\|_{F}^{2}
\hspace{3mm} \mathrm{s.t.} \hspace{1mm} \mathbf{D}_{0}=\mathbf{D}_{0}\mathbf{C}$ & $\mathbf{V}\mathcal{P}_{k}(\mathbf{\Sigma})\mathbf{V}^{T}$ & $  \left\{
   \begin{aligned}
   1 &\hspace{0.3cm} i \leq k  \\
   0 &\hspace{0.3cm} i > k  \\
   \end{aligned}
   \right.$
   & $\left\{
   \begin{aligned}
   	k&=\mathrm{argmin}_{r} r+\frac{\lambda}{2}\sum_{i>r}\sigma_{i}^{2}\\
   	r &= rank(\mathbf{D}_{0})\\
   \end{aligned}\right.$\\
$\min \frac{1}{2}\|\mathbf{C}\|_{F}^{2}+\frac{\gamma}{2}\|\mathbf{D}-\mathbf{D}\mathbf{C}\|_{F}^{2}$ & $\mathbf{V}\mathcal{P}_{\gamma}(\mathbf{\Sigma})\mathbf{V}^{T}$ & $
   \frac{\gamma\sigma_{i}^{2} }{1+\gamma \sigma_{i}^{2}} $ & Nil\\
   $\min \|\mathbf{C}\|_{\ast}+\frac{\gamma}{2}\|\mathbf{D}-\mathbf{D}\mathbf{C}\|_{F}^{2}$ & $\mathbf{V}\mathcal{P}_{\gamma}(\mathbf{\Sigma})\mathbf{V}^{T}$ & $
  \left\{
   \begin{aligned}
   1-\frac{1}{\gamma\sigma_{i}^{2}} &\hspace{0.3cm} \sigma_{i} > 1/ \sqrt{\gamma}  \\
   0                                                   &\hspace{0.3cm} \sigma_{i} \leq 1/ \sqrt{\gamma} \\
   \end{aligned}
   \right.
$ & Nil\\
$ \min \|\mathbf{C}\|_{F}+\frac{\lambda}{2}\|\mathbf{D}-\mathbf{D}_{0}\|_{F}^{2}+\frac{\gamma}{2}\|\mathbf{D}_{0}-\mathbf{D}_{0}\mathbf{C}\|_{F}^{2}$ & $\mathbf{V}\mathcal{P}_{\gamma}(\mathbf{\Sigma})\mathbf{V}^{T}$ & $
   \frac{\gamma\sigma_{i}^{2} }{1+\gamma \sigma_{i}^{2}} $ & $		\sigma_{i}=\omega_{i}+\frac{\gamma \omega_{i}}{\lambda(1+\gamma\omega_{i}^{2} )^{2}}$\\
$ \min \|\mathbf{C}\|_{\ast}+\frac{\lambda}{2}\|\mathbf{D}-\mathbf{D}_{0}\|_{F}^{2}+\frac{\gamma}{2}\|\mathbf{D}_{0}-\mathbf{D}_{0}\mathbf{C}\|_{F}^{2}$ & $\mathbf{V}\mathcal{P}_{\gamma}(\mathbf{\Sigma})\mathbf{V}^{T}$ & $
  \left\{
   \begin{aligned}
   1-\frac{1}{\gamma\omega_{i}^{2}} &\hspace{0.3cm} \omega_{i} > 1/ \sqrt{\gamma}  \\
   0                                                   &\hspace{0.3cm} \omega_{i} \leq 1/ \sqrt{\gamma} \\
   \end{aligned}
   \right.
$ &
$
  \sigma_{i}=\left\{
   \begin{aligned}
   \omega_{i}+\frac{1}{\lambda\gamma}\omega_{i}^{-3} &\hspace{0.3cm} \omega_{i} > 1/ \sqrt{\gamma}  \\
   \omega_{i}+\frac{\gamma}{\lambda}\omega_{i}                                                  &\hspace{0.3cm} \omega_{i} \leq 1/ \sqrt{\gamma} \\
   \end{aligned}
   \right.$ \\
\bottomrule
\end{tabular}
\end{footnotesize}
\end{center}
\end{table*}

\begin{table*}[t]
\caption{Connections between NNR and FNR in the case of \textit{the Laplacian noise} and \textit{the sample-specified corruption}. $\mathbf{U}\mathbf{\Sigma}\mathbf{V}^{T}=\mathbf{D}-\mathbf{E}_{t}+\alpha^{-1}_{t}\mathbf{Y}_{t}$ is the full SVD of $\mathbf{D}-\mathbf{E}_{t}+\alpha^{-1}_{t}\mathbf{Y}_{t}$, $\mathbf{\Sigma}=diag(\sigma_{1}, \sigma_{2}, \cdots)$, $\mathbf{E}_{t}$ is calculated using the augmented Lagrange multiplier method, and $\alpha_{t}$ and $\mathbf{Y}$ are ALM parameters. Note that, the Laplacian noise and the sample-specified corruption will lead to different $\mathbf{E}_{t}$.}
\label{tab2}
\begin{center}
\begin{footnotesize}
\begin{tabular}{cccc}
\toprule
Objective Function & $\mathbf{C}^{\ast}$ & $\mathcal{P}_{k}(\sigma_{i})$ or $\mathcal{P}_{\gamma}(\sigma_{i})$ & $k$ or $\omega_{i}$\\
\midrule
$\min \frac{1}{2}\|\mathbf{C}\|_{F}^{2}+\lambda\|\mathbf{D}-\mathbf{D}_{0}\|_{1}
\hspace{3mm} \mathrm{s.t.} \hspace{1mm} \mathbf{D}_{0}=\mathbf{D}_{0}\mathbf{C}$ & $\mathbf{V}\mathcal{P}_{k}(\mathbf{\Sigma})\mathbf{V}^{T}$ & $  \left\{
   \begin{aligned}
   1 &\hspace{0.3cm} i \leq k  \\
   0 &\hspace{0.3cm} i > k  \\
   \end{aligned}
   \right.$
   & $\left\{
   \begin{aligned}
   	k&=\mathrm{argmin}_{r} r+\lambda\sum_{i>r}\sigma_{i}^{2}\\
   	r &= rank(\mathbf{D}-\mathbf{E}_{t}+\alpha_{t}^{-1}\mathbf{Y}_{t})\\
   \end{aligned}\right.$\\
   $\min \|\mathbf{C}\|_{\ast}+\lambda\|\mathbf{D}-\mathbf{D}_{0}\|_{1}
\hspace{3mm} \mathrm{s.t.} \hspace{1mm} \mathbf{D}_{0}=\mathbf{D}_{0}\mathbf{C}$ & $\mathbf{V}\mathcal{P}_{k}(\mathbf{\Sigma})\mathbf{V}^{T}$ & $  \left\{
   \begin{aligned}
   1 &\hspace{0.3cm} i \leq k  \\
   0 &\hspace{0.3cm} i > k  \\
   \end{aligned}
   \right.$
   & $\left\{
   \begin{aligned}
   	k&=\mathrm{argmin}_{r} r+\frac{\lambda}{2}\sum_{i>r}\sigma_{i}^{2}\\
   	r &= rank(\mathbf{D}-\mathbf{E}_{t}+\alpha_{t}^{-1}\mathbf{Y}_{t})\\
   \end{aligned}\right.$\\
   $\min \frac{1}{2}\|\mathbf{C}\|_{F}^{2}+\lambda\|\mathbf{D}-\mathbf{D}_{0}\|_{2,1}
\hspace{3mm} \mathrm{s.t.} \hspace{1mm} \mathbf{D}_{0}=\mathbf{D}_{0}\mathbf{C}$ & $\mathbf{V}\mathcal{P}_{k}(\mathbf{\Sigma})\mathbf{V}^{T}$ & $  \left\{
   \begin{aligned}
   1 &\hspace{0.3cm} i \leq k  \\
   0 &\hspace{0.3cm} i > k  \\
   \end{aligned}
   \right.$
   & $\left\{
   \begin{aligned}
   	k&=\mathrm{argmin}_{r} r+\lambda\sum_{i>r}\sigma_{i}^{2}\\
   	r &= rank(\mathbf{D}-\mathbf{E}_{t}+\alpha_{t}^{-1}\mathbf{Y}_{t})\\
   \end{aligned}\right.$\\
   $\min \|\mathbf{C}\|_{\ast}+\lambda\|\mathbf{D}-\mathbf{D}_{0}\|_{2,1}
	\hspace{3mm} \mathrm{s.t.} \hspace{1mm} \mathbf{D}_{0}=\mathbf{D}_{0}\mathbf{C}$ & $\mathbf{V}\mathcal{P}_{k}(\mathbf{\Sigma})\mathbf{V}^{T}$ & $  \left\{
   \begin{aligned}
   1 &\hspace{0.3cm} i \leq k  \\
   0 &\hspace{0.3cm} i > k  \\
   \end{aligned}
   \right.$
   & $\left\{
   \begin{aligned}
   	k&=\mathrm{argmin}_{r} r+\frac{\lambda}{2}\sum_{i>r}\sigma_{i}^{2}\\
   	r &= rank(\mathbf{D}-\mathbf{E}_{t}+\alpha_{t}^{-1}\mathbf{Y}_{t})\\
   \end{aligned}\right.$\\
\bottomrule
\end{tabular}
\end{footnotesize}
\end{center}
\end{table*}

\section{Connections Between Nuclear Norm and Frobenius Norm Based Representation}
\label{sec2}

For a data matrix $\mathbf{D}\in \mathds{R}^{m\times n}$, let $\mathbf{D}=\mathbf{U} \mathbf{\Sigma} \mathbf{V}^{T}$ and
$\mathbf{D}=\mathbf{U}_{r} \mathbf{\Sigma}_{r} \mathbf{V}_{r}^{T}$ be the full SVD and skinny SVD of $\mathbf{D}$, where $\mathbf{\Sigma}$ and $\mathbf{\Sigma}_{r}$ are in descending order and $r$ denotes the rank of $\mathbf{D}$. $\mathbf{U}_{r}$, $\mathbf{V}_{r}$ and $\mathbf{\Sigma}_{r}$ consist of the top (\textit{i.e.}, largest) $r$ singular vectors and singular values of $\mathbf{D}$. Similar to~\cite{Wright2009,Elhamifar2013,Liu2013,Favaro2011,Vidal2014}, we assume  $\mathbf{D}=\mathbf{D}_{0}+\mathbf{E}$, where $\mathbf{D}_{0}$ denotes the clean data set and $\mathbf{E}$ denotes the errors. 

Our theoretical results will show that the optimal solutions of Frobenius-norm and nuclear-norm based objective functions are in the form of $\mathbf{C}^{\ast}=\mathbf{V}\mathcal{P}(\mathbf{\Sigma})\mathbf{V}^{T}$, where $\mathcal{P}(\cdot)$ denotes the shrinkage-thresholding operator. In other words, FNR and NNR are two solutions on the column space of $\mathbf{D}$ and they are identical in some situations. This provides a unified framework to understand FNR and NNR. The  analysis will be performed considering several popular cases including exact/relax constraint and non-corrupted/corrupted data. When the dictionary has enough representative capacity, the objective function can be formulated with the exact constraint. Otherwise, the objective function is with the relax constraint. Noticed that, the exact constraint is considerably mild since most of data sets can be naturally reconstructed by itself in practice. With the exact constraint, many methods~\cite{Favaro2011,Vidal2014} have been proposed and shown competitive performance comparing with the relax case. Besides the situation of noise-free, we will also investigate the connections between FNR and NNR when the data set contains the Gaussian noise, the Laplacian noise, or sample-specified corruption. Fig.~\ref{fig1}, Tables~\ref{tab1} and~\ref{tab2} summary our results. Noticed that, in another independent work~\cite{Pan2014}, Pan et\ al. proposed a subspace clustering method based on Frobenius norm and and reported some similar conclusions with this work. Different from this work, we mainly devote to build the theoretical connections between NNR and FNR involving different settings rather than developing new algorithm.

\begin{figure}[!t]
\begin{center}
\includegraphics[width=0.88\textwidth]{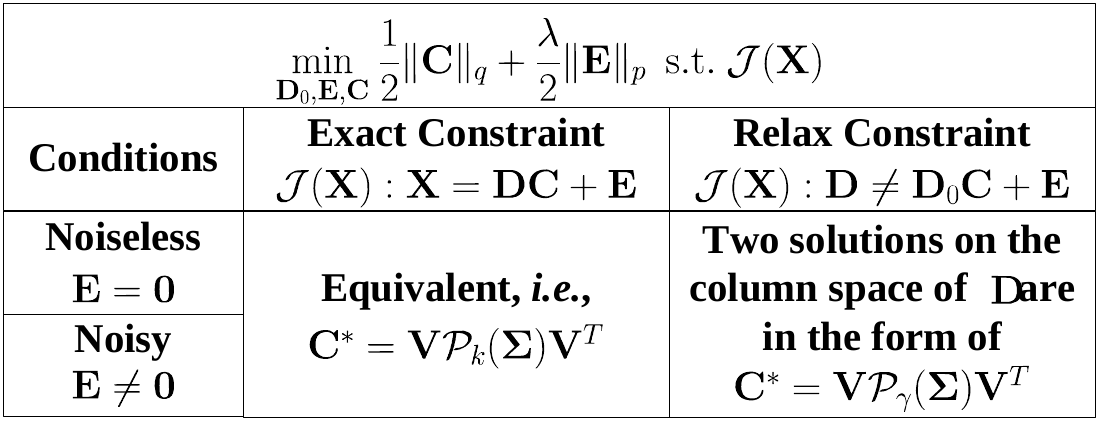}
\caption{\label{fig1} An overview of the connections between FNR ($\|\cdot\|_{q}=\|\cdot\|_{F}$) and NNR ($\|\cdot\|_{q}=\|\cdot\|_{\ast}$), where $\ell_p$ can be chosen as $\ell_1$-, $\ell_{2,1}$-, and $\ell_2$-norm corresponding to the Laplacian noise, Gaussian noise, and outliers, respectively. With the relax constraint, the major difference between NNR and FNR is the value of $\gamma$. More details are summarized in Tables~\ref{tab1} and~\ref{tab2}.}
\end{center}
\end{figure}
 
\subsection{Exact Constraint and Uncorrupted Data}

In the following analysis, we mainly focus on the case of self-expression because almost all works on NNR are carried out under such settings.

When the data set is uncorrupted and the dictionary has enough representative capacity, Liu \textit{et\ al.}~\cite{Liu2013} have shown that:

\begin{corollary}[\cite{Liu2013}]
\label{thm3}
Assume $\mathbf{D}\ne \mathbf{0}$ and $\mathbf{D}=\mathbf{DC}$ have feasible solution(s), \textit{i.e.}, $\mathbf{D}\in span(\mathbf{D})$. Then
\begin{equation}
	\label{Liu:eq1}
	\mathbf{C}^{\ast}=\mathbf{V}_{r}\mathbf{V}_{r}^{T}
\end{equation}
is the unique minimizer to 
\begin{equation}
\label{Liu:eq2}
\min \|\mathbf{C}\|_{\ast}\hspace{3mm} \mathrm{s.t.}\hspace{1mm}\mathbf{X}=\mathbf{D}\mathbf{C},
\end{equation}
 where $\mathbf{D}=\mathbf{U}_{r}\mathbf{\Sigma}_{r}\mathbf{V}^{T}_{r}$ is the skinny SVD of $\mathbf{D}$.
\end{corollary}

Considering the Frobenius norm, we can obtain the following result:
\begin{corollary}
\label{thm4}
Let $\mathbf{D}=\mathbf{U}_{r} \mathbf{\Sigma}_{r} \mathbf{V}_{r}^{T}$ be the skinny SVD of the data matrix $\mathbf{D}\ne \mathbf{0}$. The unique solution to
\begin{equation}
\label{thm1:equ1}
    \min\hspace{1mm}\|\mathbf{C}\|_{F} \hspace{3mm}
    \mathrm{s.t.} \hspace{1mm} \mathbf{D} = \mathbf{D}\mathbf{C}.
\end{equation}
is given by $\mathbf{C}^{\ast} = \mathbf{V}_{r}\mathbf{V}_{r}^{T}$, where $r$ is the rank of $\mathbf{D}$ and $\mathbf{D}$ denotes a given data set without corruptions.
\end{corollary}
\begin{proof}
Let $\mathbf{D}=\mathbf{U}_{r} \mathbf{\Sigma}_{r} \mathbf{V}_{r}^{T}$ be the skinny SVD of $\mathbf{D}$. The pseudo-inverse of $\mathbf{D}$ is $\mathbf{D}^{\dag} =
\mathbf{V}_{r} \mathbf{\Sigma}_{r}^{-1} \mathbf{U}_{r}^{T}$. By Theorem~\ref{thm2}, we obtain $\mathbf{C}^{\ast}=
\mathbf{V}_{r} \mathbf{\Sigma}_{r}^{-1} \mathbf{U}_{r}^{T}\mathbf{U}_{r} \mathbf{\Sigma}_{r} \mathbf{V}_{r}^{T}=\mathbf{V}_{r}\mathbf{V}_{r}^{T}$, as desired.
\end{proof}

From Corollaries~\ref{thm3} and \ref{thm4}, ones can find that NNR and FNR have the same optimal solution $\mathbf{V}_{r}\mathbf{V}_{r}^{T}$. This solution is also known as the shape interaction matrix~\cite{Costeira1998}. 

\subsection{Exact Constraint and Data Corrupted by Gaussian Noise}

When the data set contains Gaussian noises (\textit{i.e.}, $\mathbf{E}\ne \mathbf{0}$ and $\mathbf{E}$ is characterized by the Frobenius norm), we prove that 
\begin{equation}
\label{eq3.3}
\mathop{\min}_{\mathbf{C},\mathbf{D}_{0},\mathbf{E}}\|\mathbf{C}\|_{\ast}+\frac{\lambda}{2}\|\mathbf{E}\|_{F}^{2}
\hspace{3mm} \mathrm{s.t.} \hspace{1mm}\mathbf{D}=\mathbf{D}_{0}+\mathbf{E},\hspace{1mm}\mathbf{D}_{0}=\mathbf{D}_{0}\mathbf{C}
\end{equation}
and
\begin{equation}
\label{eq3.4}
\mathop{\min}_{\mathbf{C},\mathbf{D}_{0},\mathbf{E}}\frac{1}{2}\|\mathbf{C}\|_{F}^{2}+\frac{\lambda}{2}\|\mathbf{E}\|_{F}^{2}
\hspace{3mm} \mathrm{s.t.} \hspace{1mm}\mathbf{D}=\mathbf{D}_{0}+\mathbf{E},\hspace{1mm}\mathbf{D}_{0}=\mathbf{D}_{0}\mathbf{C}
\end{equation}
have the same minimizer in the form of $\mathbf{V}_{k}\mathbf{V}_{k}^{T}$, where $k$ is a parameter. By a simple transformation, we have the following results.
 
 \begin{theorem}[\cite{Favaro2011}]
 	\label{thm5}
 	Let $\mathbf{D}=\mathbf{U}\mathbf{\Sigma}\mathbf{V}^{T}$ be the SVD of the data matrix $\mathbf{D}$. The optimal solution to 
 	\begin{equation}
 		\mathop{\min}_{\mathbf{C},\mathbf{D}_{0}}\|\mathbf{C}\|_{\ast}+\frac{\lambda}{2}\|\mathbf{D}-\mathbf{D}_{0}\|_{F}^{2}
\hspace{3mm} \mathrm{s.t.} \hspace{1mm} \mathbf{D}_{0}=\mathbf{D}_{0}\mathbf{C},
 	\end{equation}
is given by $\mathbf{C}^{\ast}=\mathbf{V}_{k}\mathbf{V}_{k}^{T}$, where $\mathbf{\Sigma}_{k}$, $\mathbf{U}_{k}$, and $\mathbf{V}_{k}$ correspond to the top $k=\mathrm{argmin}_{r}r+\frac{\lambda}{2}\sum_{i>r}\sigma_{i}^{2}$ singular values and singular vectors of $\mathbf{D}$, respectively.
 \end{theorem}

\begin{theorem}
\label{thm6}
Let $\mathbf{D}=\mathbf{U\Sigma}\mathbf{V}^{T}$ be the full SVD of $\mathbf{D}\in \mathds{R}^{m\times n}$, where the diagonal entries of $\mathbf{\Sigma}$ are in descending order, $\mathbf{U}$ and $\mathbf{V}$ are the left and right singular vectors of $\mathbf{D}$, respectively. Suppose there exists a clean data set and errors, denoted by $\mathbf{D}_{0}$ and $\mathbf{E}$, respectively. The optimal $\mathbf{C}$ to 
\begin{equation}
\label{thm2:equ1}
\min_{\mathbf{D}_{0}, \mathbf{C}}\hspace{1mm}\frac{1}{2}\|\mathbf{C}\|_{F}^{2}+\frac{\lambda}{2}\|\mathbf{D}-\mathbf{D}_{0}\|_{F}^{2} \hspace{3mm}\mathrm{s.t.}\hspace{1mm} \mathbf{D}_{0}=\mathbf{D}_{0}\mathbf{C}
	\end{equation}
	is given by 
	\begin{equation}
		\mathbf{C}^{\ast}=\mathbf{V}\mathcal{P}_{k}(\mathbf{\Sigma})\mathbf{V}=\mathbf{V}_{k}\mathbf{V}_{k}^{T},
	\end{equation}
	where the operator $\mathcal{P}_{k}(\mathbf{\Sigma})$ performs hard thresholding on the diagonal entries of $\mathbf{\Sigma}$ by
\begin{equation}
  \mathcal{P}_{k}(\sigma_{i})=
  \left\{
   \begin{aligned}
   1 &\hspace{0.6cm} i \leq k  \\
   0 &\hspace{0.6cm} i > k  \\
   \end{aligned}
   \right.
  \end{equation}
  $\lambda$ is a balanced parameter, $k=\mathrm{argmin}_{r} r+\lambda\sum_{i>r}\sigma_{i}^{2}$, and $\sigma_{i}$ denotes the $i$-th diagonal entry of $\mathbf{\Sigma}$. \textit{i.e.}, $\mathbf{V}_{k}$ consists of the first $k$ column vectors of $\mathbf{V}$.
\end{theorem}

\begin{proof}
Let $\mathbf{D}_{0}^{\ast} $ be the optimal solution to Eq.(\ref{thm2:equ1}) and its skinny SVD be $\mathbf{D}_{0}^{\ast}=\mathbf{U}_{r}\mathbf{\Sigma}_{r}\mathbf{V}_{r}^{T}$, where $r$ is the rank of $\mathbf{D}_{0}^{\ast}$. Let $\mathbf{U}_{c}$ and $\mathbf{V}_{c}$ be the basis that orthogonal to $\mathbf{U}_{r}$ and $\mathbf{V}_{r}$, respectively. Clearly, $\mathbf{I}=\mathbf{V}_{r}\mathbf{V}_{r}^{T}+\mathbf{V}_{c}\mathbf{V}_{c}^{T}$. By Corollary~\ref{thm2}, we have $\mathbf{C}^{\ast}=\mathbf{V}_{r}\mathbf{V}_{r}^{T}$.  Next, we will bridge $\mathbf{V}_{r}$ and $\mathbf{V}$.

Use the method of Lagrange multipliers, we obtain
\begin{equation}
\label{thm2:equ3}
\mathcal{L}=\frac{1}{2}\left
\|\mathbf{C}\right\|_{F}^{2}+\frac{\lambda}{2}\left\|\mathbf{D}-\mathbf{D}_{0}\right\|_{F}^{2}+\langle\mathbf{\beta}, \mathbf{D}_{0}-\mathbf{D}_{0}\mathbf{C}\rangle,
\end{equation}
where $\mathbf{\beta}$ is the Lagrange multiplier. \\
Letting $\frac{\partial{\mathcal{L}}}{\partial{\mathbf{D}_{0}}}=0$, it gives that
\begin{equation}
\label{thm2:equ4}
\mathbf{\beta}\mathbf{V}_{c}\mathbf{V}_{c}^{T}=\lambda(\mathbf{D}-\mathbf{D}_{0}).
\end{equation}
Letting $\frac{\partial{\mathcal{L}}}{\partial{\mathbf{C}}}=0$, it gives that
\begin{equation}
\label{thm2:equ5}
\mathbf{V}_{r}\mathbf{V}_{r}^{T}=\mathbf{V}_{r}\mathbf{\Sigma}_{r}\mathbf{U}_{r}^{T}\mathbf{\beta}.
\end{equation}
Thus, $\mathbf{\beta}$ must be in the form of $\mathbf{\beta}=\mathbf{U}_{r}\mathbf{\Sigma}_{r}^{-1}\mathbf{V}_{r}^{T}+\mathbf{U}_{c}\mathbf{M}$ for some $\mathbf{M}$. Substituting this into (\ref{thm2:equ4}), it given that 
\begin{equation}
\label{thm2:equ6}
\mathbf{U}_{c}\mathbf{M}\mathbf{V}_{c}\mathbf{V}_{c}^{T}=\lambda(\mathbf{D}-\mathbf{D}_{0}).
\end{equation}
Then, we have $\|\mathbf{D}-\mathbf{D}_{0}\|_{F}^{2}=\frac{1}{\lambda^{2}}\|\mathbf{M}\mathbf{V}_{c}\|_{F}^{2}$. Since $\mathbf{V}_{c}^{T}\mathbf{V}_{c}=\mathbf{I}$, $\|\mathbf{D}-\mathbf{D}_{0}\|_{F}^{2}$ is minimized when $\mathbf{M}\mathbf{V}_{c}$ is a diagonal matrix and can be chosen as $\mathbf{M}\mathbf{V}_{c}=\mathbf{\Sigma}_{c}$. Then, $\mathbf{D}-\mathbf{D}_{0}=\frac{1}{\lambda}\mathbf{U}_{c}\mathbf{\Sigma}_{c}\mathbf{V}_{c}^{T}$. Consequently, the SVD of $\mathbf{D}$ can be rewritten as
\begin{equation}
\mathbf{D} = \mathbf{U}\mathbf{\Sigma}\mathbf{V}^{T}=
\left[                 
      \mathbf{U}_{r}\  \mathbf{U}_{c}\\  
\right]     
\left[                 
  \begin{array}{cc}   
      \mathbf{\Sigma}_{r} & \mathbf{0}\\  
      \mathbf{0} &  \frac{1}{\lambda}\mathbf{\Sigma}_{c}\\ 
     \end{array}
\right]  
\left[                 
      \mathbf{V}_{r}\  \mathbf{V}_{c}\\  
\right]^{T}    . 
\end{equation}
Thus, the minimal cost of (\ref{thm2:equ1}) is given by 
\begin{equation}
\label{thm2:equ7}
\begin{aligned}
\mathcal{L}_{\min}
&=\frac{1}{2}\|\mathbf{V}_{r}\mathbf{V}_{r}^{T}\|_{F}^{2}+\frac{\lambda}{2}\|\frac{1}{\lambda}\mathbf{\Sigma}_{c }\|_{F}^{2}\notag\\
&=\frac{1}{2} r+\frac{\lambda}{2}\sum_{i=r+1}^{\min\{m,n\}}\sigma_{i}^{2},
\end{aligned}
\end{equation}
where $\sigma_{i}$ is the $i$-th largest singular value of $\mathbf{D}$. Let $k$ be the optimal $r$, then, $k=\mathrm{argmin}_{r}r+\lambda\sum_{i>r}\sigma_{i}^{2}$. 
\end{proof}

From Theorems~\ref{thm5} and~\ref{thm6}, ones can find that the values of $k$ are slightly different. However, such difference cannot affect the equivalence conclusion because $k$ depends on the user-specified parameter $\lambda$. Moreover, this difference actually results from the constant term in our objective function.

 \subsection{Relaxed Constraint and Uncorrupted Data}

In this section, we discuss the connections between FNR and NNR when the dictionary is uncorrupted and has limited representative capacity. The objective functions are
\begin{equation}
\label{eq3.5}
\min_{\mathbf{C}}\frac{1}{2}\|\mathbf{C}\|_{F}^{2}+\frac{\gamma}{2}\|\mathbf{D}-\mathbf{D}\mathbf{C}\|_{F}^{2},
\end{equation}
and 
\begin{equation}
\label{eq3.6}
\min_{\mathbf{C}}\|\mathbf{C}\|_{\ast}+\frac{\gamma}{2}\|\mathbf{D}-\mathbf{D}\mathbf{C}\|_{F}^{2}.
\end{equation}

In a lot works such as~\cite{Lu2012,Zhang2011}, (\ref{eq3.5}) is minimized at $\mathbf{C}^{\ast}=(\mathbf{D}^{T}\mathbf{D}+\gamma\mathbf{I})^{-1}\mathbf{D}^{T}\mathbf{D}$. In this paper, we will give another form of the solution to (\ref{eq3.5}) and the new solution is performing shrinkage operation on the right eigenvectors of $\mathbf{D}$, like NNR. 

\begin{theorem}[\cite{Favaro2011}]
	\label{thm7}
	Let $\mathbf{D}=\mathbf{U}\mathbf{\Sigma}\mathbf{V}^{T}$ be the SVD of a given matrix $\mathbf{D}$. The optimal solution to 
	\begin{equation}
	\min_{\mathbf{C}}\|\mathbf{C}\|_{\ast}+\frac{\gamma}{2}\|\mathbf{D}-\mathbf{D}\mathbf{C}\|_{F}^{2}
	\end{equation}
	is
	\begin{equation}
		\mathbf{C}^{\ast}=\mathbf{V}_{1}\left(\mathbf{I}-\frac{1}{\gamma}\mathbf{\Sigma}_{1}^{-2}\right)\mathbf{V}_{1}^{T},
	\end{equation}
 where $\mathbf{U}=[\mathbf{U}_{1}\ \mathbf{U}_{2}]$, $\mathbf{\Sigma}=diag(\mathbf{\Sigma}_{1}, \mathbf{\Sigma}_{2})$, and $\mathbf{V}=[\mathbf{\mathbf{V}_{1}}\ \mathbf{V}_{2}]$ are partitioned according to the sets $\mathbf{I}_{1}=\{i:\sigma_{i}>1/\sqrt{\gamma} \}$ and $\mathbf{I}_{2}=\{i:\sigma_{i}\leq 1/\sqrt{\gamma} \}$.
\end{theorem}

\begin{theorem}
\label{thm8}
Let $\mathbf{D}=\mathbf{U\Sigma}\mathbf{V}^{T}$ be the full SVD of $\mathbf{D}\in \mathds{R}^{m\times n}$, where the diagonal entries of $\mathbf{\Sigma}$ are in descending order, $\mathbf{U}$ and $\mathbf{V}$ are corresponding left and right singular vectors, respectively. The optimal $\mathbf{C}$ to 
\begin{equation}
\label{thm6:equ1}
\min_{\mathbf{C}}\frac{1}{2}\|\mathbf{C}\|_{F}^{2}+\frac{\gamma}{2}\|\mathbf{D}-\mathbf{D}\mathbf{C}\|_{F}^{2},
\end{equation}
is given by 
\begin{equation}
\label{thm6:equ2}
	\mathbf{C}^{\ast}=\mathbf{V}\mathcal{P}_{\gamma}(\mathbf{\Sigma})\mathbf{V}^{T}=\mathbf{V}_{r}\left(\mathbf{I}-(\mathbf{I}+\gamma\mathbf{\Sigma}_{r}^{2})^{-1}\right)\mathbf{V}_{r}^{T},
\end{equation}
where $\gamma$ is a balanced factor and the operator $\mathcal{P}_{\gamma}(\mathbf{\Sigma})$ performs shrinkage-thresholding on the diagonal entries of $\mathbf{\Sigma}$ by 
 \begin{equation}
 \label{thm6:equ3}
  \mathcal{P}_{\gamma}(\sigma_{i})=
  \left\{
   \begin{aligned}
   1-\frac{1}{1+\gamma \sigma_{i}^{2}} &\hspace{0.6cm} i \leq r  \\
   0                                                   &\hspace{0.6cm} i > r  \\
   \end{aligned},
   \right.
  \end{equation}
   and $r$ is the rank of $\mathbf{D}$ and $\sigma_{i}$ denotes the $i$th diagonal entry of $\mathbf{\Sigma}$.
\end{theorem}

 \begin{proof}
	 Letting $\mathcal{L}$ denote the loss, 
      and then we have
 	\begin{equation}
 		\label{thm6:equ5}
 		\frac{\partial{\mathcal{L}}}{\partial{\mathbf{C}}} = \mathbf{C}-\gamma\mathbf{D}^{T}\mathbf{D}(\mathbf{I}-\mathbf{C}).
 	\end{equation}
 	
    Next, we will show that (\ref{thm6:equ2}) is the minimizer of $\mathcal{L}$ since $ \frac{\partial{\mathcal{L}}}{\partial{\mathbf{C}^{\ast}}}=0$.
    Letting $\mathbf{M}=(\mathbf{I}+\gamma\mathbf{\Sigma}_{r}^{2})^{-1}$ and substituting (\ref{thm6:equ2}) into (\ref{thm6:equ5}), we have  
    \begin{equation}
    \label{thm6:equ6}
	    \frac{\partial{\mathcal{L}}}{\partial{\mathbf{C}^{\ast}}} 
	    = \mathbf{V}_{r}\left(\mathbf{I}-\mathbf{M}\right)\mathbf{V}_{r}^{T}-\gamma\mathbf{D}^{T}\mathbf{D}\left(\mathbf{I}-\mathbf{V}_{r}(\mathbf{I}-\mathbf{M})\mathbf{V}_{r}^{T}\right).
    \end{equation}
    
    Let $\mathbf{V}_{r}$ and $\mathbf{V}_{c}$ be mutually orthogonal, then $\mathbf{I}=\mathbf{V}_{r}\mathbf{V}_{r}^{T}+\mathbf{V}_{c}\mathbf{V}_{c}^{T}$. Moreover, let the skinny SVD of $\mathbf{D}$ be $\mathbf{U}_{r}\mathbf{\Sigma}_{r}\mathbf{V}_{r}^{T}$, we obtain
    \begin{align}
    	\label{thm6:equ7}
    	\frac{\partial{\mathcal{L}}}{\partial{\mathbf{C}^{\ast}}} 
    	&= \mathbf{V}_{r}\mathbf{V}_{r}^{T}-\mathbf{V}_{r}\mathbf{M}\mathbf{V}_{r}^{T}-\gamma\mathbf{V}_{r}\mathbf{\Sigma}_{r}^{2}\mathbf{V}_{r}^{T}\left(\mathbf{V}_{c}\mathbf{V}_{c}^{T}+\mathbf{V}_{r}\mathbf{M}\mathbf{V}_{r}^{T}\right)  \notag\\
	    &=\mathbf{V}_{r}\mathbf{V}_{r}^{T}-\mathbf{V}_{r}\mathbf{M}\mathbf{V}_{r}^{T}-\gamma\mathbf{V}_{r}\mathbf{\Sigma}_{r}^{2}\mathbf{M}\mathbf{V}_{r}^{T}\notag\\
	    &=0
    \end{align}   
	    as desired. 
 \end{proof}
  
 \subsection{Relax Constraint and Data Corrupted by Gaussian Noise}
 
 Suppose the data set is corrupted by $\mathbf{E}$ and has limited representative capacity, the problems can be formulated as follows:
 \begin{equation}
 \label{eq3.7}
 \min_{\mathbf{C},\mathbf{D}_{0}}\|\mathbf{C}\|_{F}+\frac{\lambda}{2}\|\mathbf{D}-\mathbf{D}_{0}\|_{F}^{2}+\frac{\gamma}{2}\|\mathbf{D}_{0}-\mathbf{D}_{0}\mathbf{C}\|_{F}^{2},
 \end{equation}
and 
  \begin{equation}
  \label{eq3.8}
 \min_{\mathbf{C},\mathbf{D}_{0}}\|\mathbf{C}\|_{\ast}+\frac{\lambda}{2}\|\mathbf{D}-\mathbf{D}_{0}\|_{F}^{2}+\frac{\gamma}{2}\|\mathbf{D}_{0}-\mathbf{D}_{0}\mathbf{C}\|_{F}^{2}.
 \end{equation}

 \begin{theorem}[\cite{Favaro2011}]
 	\label{thm9}
 	Let $\mathbf{D}=\mathbf{U}\mathbf{\Sigma}\mathbf{V}^{T}$ be the SVD of the data matrix $\mathbf{D}$. The optimal solution to 
 	\begin{equation}
    \min_{\mathbf{C},\mathbf{D}_{0}}\|\mathbf{C}\|_{\ast}+\frac{\lambda}{2}\|\mathbf{D}-\mathbf{D}_{0}\|_{F}^{2}+\frac{\gamma}{2}\|\mathbf{D}_{0}-\mathbf{D}_{0}\mathbf{C}\|_{F}^{2}.
   \end{equation}
   is given by 
   \begin{equation}
   	\mathbf{C}^{\ast}=\mathbf{V}_{1}(\mathbf{I}-\frac{1}{\gamma}\mathbf{\Omega}_{1}^{-2})\mathbf{V}_{1}^{T},
   \end{equation}
 where each entry of $\mathbf{\Omega}=diag(\omega_{1},\cdots,\omega_{n})$ is obtained from one entry of $\mathbf{\Sigma}=diag(\sigma_{1},\cdots,\sigma_{n})$ as the solution to
 \begin{equation}
 \sigma_{i}= 
  \left\{
   \begin{aligned}
   \omega_{i}+\frac{1}{\lambda\gamma}\omega_{i}^{-3} &\hspace{0.3cm} \omega_{i} > 1/ \sqrt{\gamma}  \\
   \omega_{i}+\frac{\gamma}{\lambda}\omega_{i}                                                  &\hspace{0.3cm} \omega_{i} \leq 1/ \sqrt{\gamma} \\
   \end{aligned},
   \right.
 \end{equation}
that minimizes the cost, and the matrices $\mathbf{U}=[\mathbf{U}_{1}\ \mathbf{U}_{2}]$, $\mathbf{\Omega}=diag(\mathbf{\omega}_{1}, \mathbf{\omega}_{2})$, and $\mathbf{V}=[\mathbf{V}_{1}\ \mathbf{V}_{2}]$ are partitioned according to the sets $\mathbf{I}_{1}=\{i:\omega_{i}>1/\sqrt{\gamma} \}$ and $\mathbf{I}_{2}=\{i:\omega_{i}\leq 1/\sqrt{\gamma} \}$.

 \end{theorem}

\begin{theorem}
	\label{thm10}
	 Let $\mathbf{D}=\mathbf{U}_{r}\mathbf{\Sigma}_{r}\mathbf{V}_{r}$ be the skinny SVD of $\mathbf{D}\in \mathds{R}^{m\times n}$, where $r$ denotes the rank of $\mathbf{D}$ and the diagonal entries of $\mathbf{\Omega}_{r}$ is in descending order. The optimal solutions to
	\begin{equation}
		\label{thm4:equ1}
		 \min_{\mathbf{C},\mathbf{D}_{0}}\|\mathbf{C}\|_{F}+\frac{\lambda}{2}\|\mathbf{D}-\mathbf{D}_{0}\|_{F}^{2}+\frac{\gamma}{2}\|\mathbf{D}_{0}-\mathbf{D}_{0}\mathbf{C}\|_{F}^{2},
	\end{equation}
	are given by 
	\begin{equation}
		\label{thm4:equ2}
		\mathbf{D}_{0}^{\ast}=\mathbf{U}_{r}\mathbf{\Omega}_{r}\mathbf{V}_{r}^{T},
	\end{equation}
	and 
	\begin{equation}
		\label{thm4:equ3}
		\mathbf{C}^{\ast}
		=\mathbf{V}_{r}\left(\mathbf{I}-(\mathbf{I}+\gamma\mathbf{\Omega}_{r}^{2})^{-1}\right)\mathbf{V}_{r}^{T},
	\end{equation}
	where $\sigma_{i}$ and $\omega_{i}$ are the diagonal entries on $\mathbf{\Sigma}_{r}$ and $\Omega_{r}$, respectively. 
	\begin{equation}
	\label{thm4:equ4}
		\sigma_{i}=\omega_{i}+\frac{\gamma \omega_{i}}{\lambda(1+\gamma\omega_{i}^{2} )^{2}}.
	\end{equation}

\end{theorem}
\begin{proof}
	 Letting $\mathcal{L}$ denote the loss, 		
	 we have
	\begin{equation}
		\label{thm4:equ6}
		\frac{\partial{\mathcal{L}}}{\partial{\mathbf{D}_{0}}}=-\lambda(\mathbf{D}-\mathbf{D}_{0})+\gamma\mathbf{D}_{0}(\mathbf{I}-\mathbf{C}^{\ast})(\mathbf{I}-\mathbf{C}^{\ast})^{T}.
	\end{equation}
	
	Next, we will bridge $\mathbf{D}$ and $\mathbf{D}_{0}$.
	Let $\mathbf{D}_{0}=\mathbf{U}_{r}\mathbf{\Omega}_{r}\mathbf{V}_{r}^{T}$ be the skinny SVD of $\mathbf{D}_{0}$. From Theorem~\ref{thm6}, we have $\mathbf{C}^{\ast}
		=\mathbf{V}_{r}\left(\mathbf{I}-(\mathbf{I}+\gamma\mathbf{\Omega}_{r}^{2})^{-1}\right)\mathbf{V}_{r}^{T}$.  Substituting this into (\ref{thm4:equ6}) and letting $\frac{\partial{\mathcal{L}}}{\partial{\mathbf{D}_{0}}}=0$, we obtain
		\begin{equation}
			\label{thm4:equ7}
			\mathbf{D} = \mathbf{U}_{r}\left(\mathbf{\Omega}_{r}+\frac{\gamma}{\lambda}\mathbf{\Omega}_{r}(\mathbf{I}+\gamma\mathbf{\Omega}_{r}^{2} )^{-2})  \right)
			\mathbf{V}_{r}^{T},
		\end{equation}
	which is a valid SVD of $\mathbf{D}$. 	
\end{proof}

Theorems~\ref{thm7}--\ref{thm10} establish the relationships between FNR and NNR in the case of the limited representative capacity. Although FNR and NNR are not identical in such settings, they can be unified into a framework, \textit{i.e.}, both two methods obtain a solution from the column space of $\mathbf{D}$. The major difference between them is the adopted scaling factor. Moreover, NNR and FNR will truncates the trivial entries of coefficients in the case of uncorrupted data. With respect to corrupted case, two methods only scales the self-expressive coefficients by performing shrinkage.

\subsection{Exact Constraint and Data Corrupted by Laplacian Noise}

The above analysis are based on the noise-free or the  Gaussian noise assumptions. In this section, we investigate the Laplacian noise situation with the exact constraint. More specifically, we will prove that the optimal solutions to

\begin{equation}
\label{eqb3.1}
\mathop{\min}_{\mathbf{C},\mathbf{D}_{0},\mathbf{E}}\|\mathbf{C}\|_{\ast}+\lambda\|\mathbf{E}\|_{1}
\hspace{3mm} \mathrm{s.t.} \hspace{1mm}\mathbf{D}=\mathbf{D}_{0}+\mathbf{E},\hspace{1mm}\mathbf{D}_{0}=\mathbf{D}_{0}\mathbf{C}
\end{equation}
and
\begin{equation}
\label{eqb3.2}
\mathop{\min}_{\mathbf{C},\mathbf{D}_{0},\mathbf{E}}\frac{1}{2}\|\mathbf{C}\|_{F}^{2}+\lambda\|\mathbf{E}\|_{1}
\hspace{3mm} \mathrm{s.t.} \hspace{1mm}\mathbf{D}=\mathbf{D}_{0}+\mathbf{E},\hspace{1mm}\mathbf{D}_{0}=\mathbf{D}_{0}\mathbf{C}
\end{equation}
have the same form. 

As $\mathbf{D}_{0}$ is unknown and $\ell_1$-norm has no closed-form solution, we can solve Eqs.(\ref{eqb3.1}) and (\ref{eqb3.2}) using the augmented Lagrange multiplier method (ALM)~\cite{Lin2011NIPS}.

\begin{proposition}[\cite{Favaro2011}]
	\label{thm11}
	The optimal solution to Eq.(\ref{eqb3.1}) is given by 
	\begin{equation}
	\label{thm11:equ1}
		\mathbf{C}^{\ast}=\mathbf{V}_{k}\mathbf{V}_{k}^{T},
	\end{equation}
	where $k=\mathrm{argmin}_{r} r+\frac{\lambda}{2}\sum_{i>r}\sigma_{i}^{2}$, $\mathbf{V}_{k}$ consists of the first $k$ column vectors of $\mathbf{V}$, and $\mathbf{V}$ is iteratively computed via the following updated rules: 
	\begin{equation}
		\label{thm11:equ2}
		\mathbf{U}\mathbf{\Sigma}\mathbf{V}^{T}=\mathbf{D}-\mathbf{E}_{t}+\alpha^{-1}_{t}\mathbf{Y}_{t}
	\end{equation}
	\begin{equation}
		\label{thm11:equ3}
			\mathbf{D}_{0_{t+1}}=\mathbf{U}\mathcal{P}_{k}(\mathbf{\Sigma})\mathbf{V}^{T}	
	\end{equation}
	\begin{equation}
		\label{thm11:equ4}
		\mathbf{E}_{t+1}=\mathcal{S}_{\gamma\alpha^{-1}}(\mathbf{D}-\mathbf{D}_{0_{t+1}}+\alpha_{t}^{-1}\mathbf{Y}_{t})
	\end{equation}
	\begin{equation}
		\label{thm11:equ5}
		\mathbf{Y}_{k+1}=\mathbf{Y}_{k}+\alpha_{k}(\mathbf{D}-\mathbf{D}_{0_{t+1}}-\mathbf{E}_{t+1})
	\end{equation}
	\begin{equation}
		\label{thm11:equ6}
		\alpha_{t+1}=\rho\alpha_{t},
	\end{equation}
where $\rho>1$ is the learning rate of ALM and $\mathcal{S}$ is a shrinkage-thresholding operator
\begin{equation}
	\label{thm11:equ7}
	\mathcal{S}_{\epsilon}(x)=
	\left\{
   \begin{aligned}
   x-\epsilon  &\hspace{0.3cm} x > \epsilon \\
   x+\epsilon &\hspace{0.3cm} x<-\epsilon \\
   0               &\hspace{0.3cm} else\\
   \end{aligned}
   \right.
\end{equation}	 	
\end{proposition}

\begin{proposition}
\label{thm12}
The optimal solution to Eq.(\ref{eqb3.2}) is given by
	\begin{equation}
	\label{thm12:equ1}
		\mathbf{C}^{\ast}=\mathbf{V}_{k}\mathbf{V}_{k}^{T},
	\end{equation}
	where $\mathbf{V}_{k}$ consists of the first $k$ column vectors of $\mathbf{V}$, and the updated rules are 
	\begin{equation}
		\label{thm12:equ2}
		\mathbf{U}\mathbf{\Sigma}\mathbf{V}^{T}=\mathbf{D}-\mathbf{E}_{t}+\alpha^{-1}_{t}\mathbf{Y}_{t}
	\end{equation}
	\begin{equation}
		\label{thm12:equ3}
			\mathbf{D}_{0_{t+1}}=\mathbf{U}\mathcal{P}_{k}(\mathbf{\Sigma})\mathbf{V}^{T}	
	\end{equation}
	\begin{equation}
		\label{thm12:equ4}
		\mathbf{E}_{t+1}=\mathcal{S}_{\lambda\alpha^{-1}}(\mathbf{D}-\mathbf{D}_{0_{t+1}}+\alpha_{t}^{-1}\mathbf{Y}_{t})
	\end{equation}
	\begin{equation}
		\label{thm12:equ5}
		\mathbf{Y}_{k+1}=\mathbf{Y}_{k}+\alpha_{k}(\mathbf{D}-\mathbf{D}_{0_{t+1}}-\mathbf{E}_{t+1})
	\end{equation}
	\begin{equation}
		\label{thm12:equ6}
		\alpha_{t+1}=\rho\alpha_{t},
	\end{equation}
\end{proposition}

\begin{proof}
	Using the augmented Lagrangian formulation, Eq.(\ref{eqb3.2}) can be rewritten as
	
	\begin{align}
		\label{thm12:equ7}
		\min \hspace{1mm}&\frac{1}{2}\|\mathbf{C}\|_{F}^{2}+\lambda\|\mathbf{E}\|_{1}+\frac{\alpha}{2}\|\mathbf{D}-\mathbf{D}_{0}-\mathbf{E}\|_{F}^{2}\notag\\
		&\hspace{3cm} +\langle\mathbf{Y},\mathbf{D}-\mathbf{D}_{0}-\mathbf{E}\rangle\notag\\
		\mathrm{s.t.}\hspace{3mm}&\mathbf{D}_{0}=\mathbf{D}_{0}\mathbf{C}. 
	\end{align}
	
	By fixing others, we obtain $\mathbf{D}_{0}^{\ast}$ by solving 
	\begin{equation}
		\label{thm12:equ8}
				\min\frac{\alpha}{2}\|\mathbf{D}-\mathbf{E}+\alpha^{-1}\mathbf{Y}	-\mathbf{D}_{0}\|_{F}^{2}\hspace{3mm}\mathrm{s.t.}\hspace{1mm}\mathbf{D}_{0}=\mathbf{D}_{0}\mathbf{C}
	\end{equation}	
	
	According to Theorem~\ref{thm6}, the optimal solutions to Eq.(\ref{thm12:equ8}) is given by $\mathbf{D}_{0}^{\ast}=\mathbf{U}_{k}\mathbf{\Sigma}_{k}\mathbf{V}_{k}^{k}$ and $\mathbf{C}^{\ast}=\mathbf{V}_{k}\mathbf{V}_{k}^{T}$, where $\mathbf{V}_{k}$ consists of the first $k$ right singular vectors of $\mathbf{D}-\mathbf{E}+\alpha^{-1}\mathbf{Y}$, $k=\mathrm{argmin}_{r} r+\lambda\sum_{i>r}\sigma_{i}^{2}$. Therefore, the optimal solutions to Eq.(\ref{eqb3.2}) can be iteratively computed via Eqs.(\ref{thm12:equ2})--(\ref{thm12:equ6}).	
\end{proof}

From Propositions 1--2, ones can find that the updated rules of NNR and FNR are identical under the framework of ALM. This would lead to the same minimizer to NNR and FNR. 

\subsection{Exact Constraint and Data Corrupted by Sample-specified Noise}

Besides the Gaussian noise and the Laplacian noise, we investigate sample-specified corruptions such as outliers~\cite{Liu2013,Nie2010,Cai2010} by adopting the $\ell_{2,1}$ norm. The formulations are as follows:

\begin{equation}
\label{eqb3.3}
\mathop{\min}_{\mathbf{C},\mathbf{D}_{0},\mathbf{E}}\|\mathbf{C}\|_{\ast}+\lambda\|\mathbf{E}\|_{2,1}
\hspace{3mm} \mathrm{s.t.} \hspace{1mm}\mathbf{D}=\mathbf{D}_{0}+\mathbf{E},\hspace{1mm}\mathbf{D}_{0}=\mathbf{D}_{0}\mathbf{C}
\end{equation}
and
\begin{equation}
\label{eqb3.4}
\mathop{\min}_{\mathbf{C},\mathbf{D}_{0},\mathbf{E}}\frac{1}{2}\|\mathbf{C}\|_{F}^{2}+\lambda\|\mathbf{E}\|_{2,1}
\hspace{3mm} \mathrm{s.t.} \hspace{1mm}\mathbf{D}=\mathbf{D}_{0}+\mathbf{E},\hspace{1mm}\mathbf{D}_{0}=\mathbf{D}_{0}\mathbf{C}
\end{equation}

Similar to Propositions 1 and 2, it is easy to show that the optimal solutions to Eqs.(\ref{eqb3.3})--(\ref{eqb3.4}) can be calculated via 
\begin{equation}
		\label{thm13:equ1}
		\mathbf{U}\mathbf{\Sigma}\mathbf{V}^{T}=\mathbf{D}-\mathbf{E}_{t}+\alpha^{-1}_{t}\mathbf{Y}_{t}
	\end{equation}
	\begin{equation}
		\label{thm13:equ2}
			\mathbf{D}_{0_{t+1}}=\mathbf{U}\mathcal{P}_{k}(\mathbf{\Sigma})\mathbf{V}^{T}	
	\end{equation}
	\begin{equation}
		\label{thm13:equ3}
		\mathbf{E}_{t+1}=\mathcal{Q}_{\lambda\alpha^{-1}}(\mathbf{D}-\mathbf{D}_{0_{t+1}}+\alpha_{t}^{-1}\mathbf{Y}_{t})
	\end{equation}
	\begin{equation}
		\label{thm13:equ4}
		\mathbf{Y}_{k+1}=\mathbf{Y}_{k}+\alpha_{k}(\mathbf{D}-\mathbf{D}_{0_{t+1}}-\mathbf{E}_{t+1})
	\end{equation}
	\begin{equation}
		\label{thm13:equ5}
		\alpha_{t+1}=\rho\alpha_{t},
	\end{equation}
where the operator $\mathcal{Q}_{\epsilon}(\mathbf{X})$ is defined on the column of $\mathbf{X}$, \textit{i.e.},
\begin{equation}
	\label{thm13:equ6}
	\mathcal{Q}_{\epsilon}([\mathbf{X}]_{:,i})=
	\left\{
   \begin{aligned}
   \frac{\|[\mathbf{X}]_{:,i}\|_{2}-\epsilon}{\|[\mathbf{X}]_{:,i}\|_{2}}  &\hspace{0.3cm} \|[\mathbf{X}]_{:,i}\|_{2} > \epsilon \\
   0               &\hspace{0.3cm} otherwise\\
   \end{aligned}
   \right.
\end{equation}
where $[\mathbf{X}]_{:,i}$ denotes the $i$th column of $\mathbf{X}$.

Thus, ones can find that the optimal solutions of FNR and NNR are with the same form. The only one difference between them is the value of $k$, \textit{i.e.}, $k=\mathrm{argmin}_{r}+\frac{\lambda}{2}\sum_{i>r}\sigma_{i}^{2}$ for NNR and $k=\mathrm{argmin}_{r}+\lambda\sum_{i>r}\sigma_{i}^{2}$ for FNR. 

With respect to the relax constraint, ones can also establish the connections between FNR and NNR considering the Laplacian noise and sample-specified corruption. The analysis will be based on Theorem~\ref{thm6} and the form of $\mathbf{C}^{\ast}$ is similar to the case of the exact constraint. 

\section{Discussions}

In this section, we first give the computational complexity analysis for FNR in different settings and then discuss the advantages of FNR over NNR in application scenario.  

The above analysis shows that FNR and NNR are with the same form of solution. Thus, we can easily conclude that their computational complexity are the same under the same setting. More specifically, 1) when the input is free to corruption or contaminated by Gaussian noise, FNR and NNR will take $O(m^{2}n+n^{3})$ to perform SVD on the input and then use $nk^{2}$ to obtain the representation; 2) when the input contains Laplacian noise or sample-specified corruption, FNR and NNR will take $O(tnm^{2}+tn^{3})$ to iteratively obtain the SVD of the input and then use $nk^{2}$ to obtain the representation.

Our analysis explicitly gives the connections between NNR and FNR in theory. Thus, ones may hope to further understand them in the context of application scenario based on the theoretical analysis. Referring to experimental studies in  existing works, we could conclude that: 1) for face recognition task, FNR would be more competitive since it could achieve comparable performance with over hundred times speedup as shown in~\cite{Naseem2010,Shi2011face,Zhang2011,Peng2014}; 2) when dictionary can exactly reconstruct the input, both our theoretical and experimental analysis show that FNR 	and NNR perform comparable in feature extraction~\cite{Peng2016:Auto}, image clustering and motion segmentation~\cite{Favaro2011,Vidal2014}; 3) otherwise, FNR is better than NNR for feature extraction~\cite{Peng2012}, image clustering and motion segmentation~\cite{Peng2015robust,Lu2012,Liu2013}.

\section{Conclusion}
\label{sec6}

In this paper, we investigated the connections between FNR and NNR in the case of the exact and the relax constraint. When the objective function is with the exact constraint, FNR is exactly NNR even though the data set contains the Gaussian noise, Laplacian noise, or sample-specified corruption. In the case of the relax constraint, FNR and NNR are two solutions on the column space of inputs. Under such a  setting, the only one difference between FNR and NNR is the value of the thresholding parameter $\gamma$. Our theoretical results is complementary and a small step forward to existing compressive sensing. The major difference is that this work establishes the connections between the convex problem caused by $\ell_1$-norm and the strictly convex problem caused by $\ell_2$-norm in matrix space, while compressive sensing focuses on the equivalence between the non-convex problem caused by $\ell_0$-norm and convex problem caused by $\ell_1$-norm. 

%

%
%
%
%


%


\ifCLASSOPTIONcaptionsoff
  \newpage
\fi



%

%
%

\bibliography{Equivalence}

\begin{thebibliography}{10}
\providecommand{\url}[1]{#1}
\csname url@samestyle\endcsname
\providecommand{\newblock}{\relax}
\providecommand{\bibinfo}[2]{#2}
\providecommand{\BIBentrySTDinterwordspacing}{\spaceskip=0pt\relax}
\providecommand{\BIBentryALTinterwordstretchfactor}{4}
\providecommand{\BIBentryALTinterwordspacing}{\spaceskip=\fontdimen2\font plus
\BIBentryALTinterwordstretchfactor\fontdimen3\font minus
  \fontdimen4\font\relax}
\providecommand{\BIBforeignlanguage}[2]{{%
\expandafter\ifx\csname l@#1\endcsname\relax
\typeout{** WARNING: IEEEtran.bst: No hyphenation pattern has been}%
\typeout{** loaded for the language `#1'. Using the pattern for}%
\typeout{** the default language instead.}%
\else
\language=\csname l@#1\endcsname
\fi
#2}}
\providecommand{\BIBdecl}{\relax}
\BIBdecl

\bibitem{Donoho2003}
D.~L. Donoho and M.~Elad, ``Optimally sparse representation in general
  (nonorthogonal) dictionaries via $\ell_1$ minimization,'' \emph{Proc. Natl.
  Acad. Sci.}, vol. 100, no.~5, pp. 2197--2202, 2003.

\bibitem{Wright2010}
J.~Wright, Y.~Ma, J.~Mairal, G.~Sapiro, T.~Huang, and S.~Yan, ``Sparse
  representation for computer vision and pattern recognition,'' \emph{Proc.
  IEEE}, vol.~98, no.~6, pp. 1031--1044, Jun. 2010.

\bibitem{Aharon2006}
M.~Aharon, M.~Elad, and A.~Bruckstein, ``The {K-SVD}: An algorithm for
  designing overcomplete dictionaries for sparse representation,'' \emph{IEEE
  T. Signal Process.}, vol.~54, no.~11, pp. 4311--4322, Nov 2006.

\bibitem{Wright2009}
J.~Wright, A.~Y. Yang, A.~Ganesh, S.~S. Sastry, and Y.~Ma, ``Robust face
  recognition via sparse representation,'' \emph{IEEE T. Pattern Anal. Mach.
  Intell.}, vol.~31, no.~2, pp. 210--227, 2009.

\bibitem{Xu2013:TNNLS}
X.~Xu, Z.~Hou, C.~Lian, and H.~He, ``Online learning control using adaptive
  critic designs with sparse kernel machines,'' \emph{IEEE T. Neural Netw.
  Learn. Syst.}, vol.~24, no.~5, pp. 762--775, May 2013.

\bibitem{Cheng2010}
B.~Cheng, J.~Yang, S.~Yan, Y.~Fu, and T.~Huang, ``Learning with {L}1-graph for
  image analysis,'' \emph{IEEE T. Image Process.}, vol.~19, no.~4, pp.
  858--866, 2010.

\bibitem{Elhamifar2013}
E.~Elhamifar and R.~Vidal, ``Sparse subspace clustering: Algorithm, theory, and
  applications,'' \emph{IEEE T. Pattern Anal. Mach. Intell.}, vol.~35, no.~11,
  pp. 2765--2781, 2013.

\bibitem{Peng2015SRSC}
X.~Peng, H.~Tang, L.~Zhang, Z.~Yi, and S.~Xiao, ``A unified framework for
  representation-based subspace clustering of out-of-sample and large-scale
  data,'' \emph{IEEE T. Neural Netw. Learn. Syst.}, vol.~PP, no.~99, pp. 1--14,
  2015.

\bibitem{Liu2013}
G.~Liu, Z.~Lin, S.~Yan, J.~Sun, Y.~Yu, and Y.~Ma, ``Robust recovery of subspace
  structures by low-rank representation,'' \emph{IEEE T. Pattern Anal. Mach.
  Intell.}, vol.~35, no.~1, pp. 171--184, 2013.

\bibitem{Favaro2011}
P.~Favaro, R.~Vidal, and A.~Ravichandran, ``A closed form solution to robust
  subspace estimation and clustering,'' in \emph{Proc. of 24th IEEE Conf.
  Comput. Vis. and Pattern Recognit.}, Colorado Springs, CO, Jun. 2011, pp.
  1801--1807.

\bibitem{Vidal2014}
R.~Vidal and P.~Favaro, ``Low rank subspace clustering ({LRSC}),''
  \emph{Pattern Recognit. Lett.}, vol.~43, no.~0, pp. 47 -- 61, 2014.

\bibitem{Sprechmann2015}
P.~Sprechmann, A.~Bronstein, and G.~Sapiro, ``Learning efficient sparse and low
  rank models,'' \emph{IEEE T. Pattern Anal. Mach. Intell.}, vol.~37, no.~9,
  pp. 1821--1833, Sept 2015.

\bibitem{Xiao2015TNN}
S.~Xiao, M.~Tan, and D.~Xu, ``Robust kernel low rank representation,''
  \emph{IEEE T. Neural Netw. Learn. Syst.}, vol.~PP, no.~99, pp. 1--1, 2015.

\bibitem{Xiao2015aTNN}
S.~Xiao, D.~Xu, and J.~Wu, ``Automatic face naming by learning discriminative
  affinity matrices from weakly labeled images,'' \emph{IEEE T. Neural Netw.
  Learn. Syst.}, vol.~26, no.~10, pp. 2440--2452, Oct. 2015.

\bibitem{Cai2010}
J.-F. Cai, E.~J. Cand{\`e}s, and Z.~Shen, ``A singular value thresholding
  algorithm for matrix completion,'' \emph{{SIAM} J. Optim.}, vol.~20, no.~4,
  pp. 1956--1982, 2010.

\bibitem{Naseem2010}
I.~Naseem, R.~Togneri, and M.~Bennamoun, ``Linear regression for face
  recognition,'' \emph{IEEE T. Pattern Anal. Mach. Intell.}, vol.~32, no.~11,
  pp. 2106--2112, Nov. 2010.

\bibitem{Shi2011face}
Q.~Shi, A.~Eriksson, A.~Van Den~Hengel, and C.~Shen, ``Is face recognition
  really a compressive sensing problem?'' in \emph{Proc. of 24th IEEE Conf.
  Comput. Vis. and Pattern Recognit.}, Colorado, Springs, Jun. 2011, pp.
  553--560.

\bibitem{Zhang2011}
L.~Zhang, M.~Yang, and X.~Feng, ``Sparse representation or collaborative
  representation: Which helps face recognition?'' in \emph{Proc. the 13th Int.
  Conf. on Comput. Vis.}, Barcelona, Spain, Nov. 2011, pp. 471--478.

\bibitem{Peng2012}
X.~Peng, Z.~Yu, Z.~Yi, and H.~Tang, ``Constructing the l2-graph for robust
  subspace learning and subspace clustering,'' \emph{IEEE T. Cybern.}, vol.~PP,
  no.~99, pp. 1--14, 2016.

\bibitem{Peng2016:Auto}
X.~Peng, J.~Lu, Z.~Yi, and R.~Yan, ``Automatic subspace learning via principal
  coefficients embedding,'' \emph{IEEE T. Cybern.}, vol.~PP, no.~99, pp. 1--14,
  2016.

\bibitem{Xu2016:TNNLS}
X.~Xu, Z.~Huang, L.~Zuo, and H.~He, ``Manifold-based reinforcement learning via
  locally linear reconstruction,'' \emph{IEEE T. Neural Netw. Learn. Syst.},
  vol.~PP, no.~99, pp. 1--14, 2016.

\bibitem{Lu2012}
C.-Y. Lu, H.~Min, Z.-Q. Zhao, L.~Zhu, D.-S. Huang, and S.~Yan, ``Robust and
  efficient subspace segmentation via least squares regression,'' in
  \emph{Proc. of 12th Eur. Conf. Computer Vis.}, Firenze, Italy, Oct. 2012, pp.
  347--360.

\bibitem{Peng2015robust}
X.~Peng, Z.~Yi, and H.~Tang, ``Robust subspace clustering via thresholding
  ridge regression,'' in \emph{Proc. of 29th AAAI Conf. Artif. Intell.}, Austin
  Texas, USA, Jan. 2015, pp. 3827--3833.

\bibitem{Boyd2011}
S.~Boyd, N.~Parikh, E.~Chu, B.~Peleato, and J.~Eckstein, ``Distributed
  optimization and statistical learning via the alternating direction method of
  multipliers,'' \emph{Found. and Trends in Machine Learn.}, vol.~3, no.~1, pp.
  1--122, Jan. 2011.

\bibitem{Zhang2014}
H.~Zhang, Z.~Yi, and X.~Peng, ``{fLRR}: fast low-rank representation using
  frobenius-norm,'' \emph{Electron. Lett.}, vol.~50, no.~13, pp. 936--938, June
  2014.

\bibitem{Pan2014}
P.~Ji, M.~Salzmann, and H.~Li, ``Efficient dense subspace clustering,'' in
  \emph{Proc. of 14th IEEE Winter Conf. Appl. of Computer Vis.}, Springs, CO,
  Mar. 2014, pp. 461--468.

\bibitem{Costeira1998}
J.~P. Costeira and T.~Kanade, ``A multibody factorization method for
  independently moving objects,'' \emph{Int. J. Comput. Vis.}, vol.~29, no.~3,
  pp. 159--179, 1998.

\bibitem{Lin2011NIPS}
Z.~Lin, R.~Liu, and Z.~Su, ``Linearized alternating direction method with
  adaptive penalty for low-rank representation,'' in \emph{Proc. of 24th Adv.
  in Neural Inf. Process. Syst.}, Granada, Spain, Dec. 2011, pp. 612--620.

\bibitem{Nie2010}
F.~Nie, H.~Huang, X.~Cai, and C.~H. Ding, ``Efficient and robust feature
  selection via joint $\ell_{2,1}$-norms minimization,'' in \emph{Proc. of 23th
  Adv. in Neural Inf. Process. Syst.}, Harrahs and Harveys, Lake Tahoe, Dec.
  2010, pp. 1813--1821.

\bibitem{Peng2014}
X.~Peng, L.~Zhang, Z.~Yi, and K.~K. Tan, ``Learning locality-constrained
  collaborative representation for robust face recognition,'' \emph{Pattern
  Recognition}, vol.~47, no.~9, pp. 2794--2806, 2014.

\end{thebibliography}
\bibliographystyle{IEEEtran}
 
%
%
%
%
%
%
%
%

\end{document}